\documentclass{article} % For LaTeX2e
\usepackage{iclr2026_delta,times}

\usepackage{amsmath,amsfonts,bm}

\newcommand{\figleft}{{\em (Left)}}
\newcommand{\figcenter}{{\em (Center)}}
\newcommand{\figright}{{\em (Right)}}
\newcommand{\figtop}{{\em (Top)}}
\newcommand{\figbottom}{{\em (Bottom)}}

\def\Propref#1{Proposition~\ref{#1}}
\def\Figref#1{Figure~\ref{#1}}

\def\Appref#1{Appendix~\ref{#1}}
\def\Secref#1{Section~\ref{#1}}
\def\eqref#1{equation~\ref{#1}}
\def\Eqref#1{Equation~\ref{#1}}
\def\Tabref#1{Table~\ref{#1}}

\def\1{\bm{1}}

\def\vzero{{\bm{0}}}

\def\vk{{\bm{k}}}

\def\mI{{\bm{I}}}

\DeclareMathAlphabet{\mathsfit}{\encodingdefault}{\sfdefault}{m}{sl}
\SetMathAlphabet{\mathsfit}{bold}{\encodingdefault}{\sfdefault}{bx}{n}

\newcommand{\E}{\mathbb{E}}

\newcommand{\R}{\mathbb{R}}

\usepackage{hyperref}
\usepackage{url}

\usepackage{pifont}
\usepackage{booktabs}
\usepackage{mathtools}
\usepackage{caption}
\usepackage{amsthm}

\newif\ifshowcomments
\showcommentstrue

\ifshowcomments
  \newcommand{\inlineComment}[3]{\textcolor{#2}{[#3 -\textit{#1}]}}
  \newcommand{\shox}[1]{\inlineComment{Shox}{magenta}{#1}}
\else
  \newcommand{\inlineComment}[3]{}
  \newcommand{\shox}[1]{}
\fi

\newtheorem{theorem}{Theorem}[section]

\newtheorem{proposition}[theorem]{Proposition}

\iclrfinalcopy

\title{Low-Pass Flow Matching}

\author{Francesco M. Ruscio \\
ELLIS Institute T\"ubingen \& \\
Max Planck Institute for Intelligent Systems \&\\
AITHYRA\\
\texttt{fruscio@aithyra.at} \\ \\
\And
T. Konstantin Rusch \\
ELLIS Institute T\"ubingen \&\\
Max Planck Institute for Intelligent Systems \& \\
T\"ubingen AI Center \& \\
Liquid AI \\
\texttt{tkrusch@tue.ellis.eu} 
\\ \\
}

\begin{document}

\maketitle

\begin{abstract}
Flow Matching typically relies on white noise sources, a choice often misaligned with the power spectra of natural data, which tend to decay with frequency. To address this, we introduce \textbf{Low-Pass Flow Matching}, a variant of Flow Matching based on an operator-modulated interpolant. This formulation induces a time-varying spectral bias that transitions from the source spectrum to a frequency-decaying bias as the path approaches the data. We validate our method on unconditional image generation tasks, including the scientific Galaxy10 dataset. Empirically, we show that our method is particularly effective when paired with adaptive ODE solvers,  where it improves or preserves sample quality while substantially reducing sampling cost compared to standard baselines.
\end{abstract}
\begin{figure}[h]
    \centering
    \includegraphics[width=0.58\linewidth]{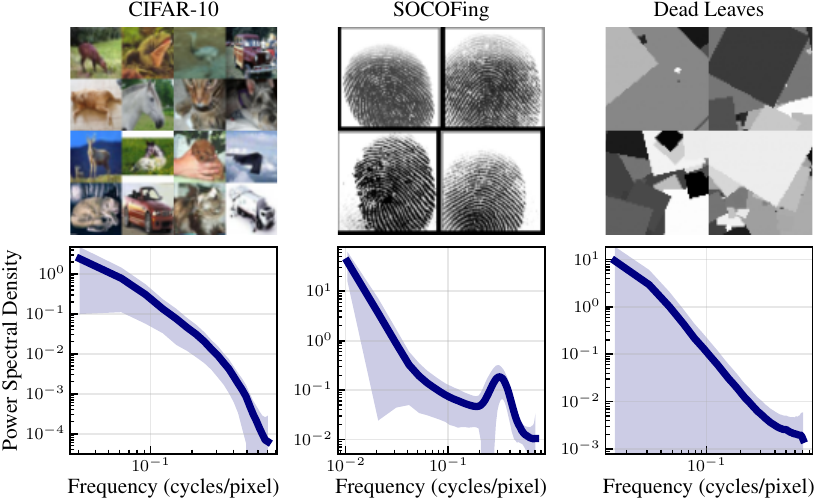}
    \caption{Radially averaged power spectral density of the training set across datasets. 
    \figleft CIFAR-10. 
    \figcenter Sokoto Coventry Fingerprint dataset \citep{shehu2018sokotocoventryfingerprintdataset}. 
    \figright Dead Leaves synthetic dataset \citep{baradad2021learning}.}
    \label{fig:1}
\end{figure}
\section{Introduction}
Dynamical systems have emerged as a powerful framework for generative modeling \citep{lai2025principlesdiffusionmodels, lipman2024flowmatchingguidecode}. Modern approaches rely on SDEs \citep{ho2020denoising, song2021denoising,  song2021scorebased}, ODEs \citep{albergo2023building, lipman2023flow, liu2023flow}, and flow maps \citep{boffi2025how}, achieving substantial progress in vision \citep{gupta2024photorealistic, rombach2022high}, audio \citep{audio}, and scientific domains \citep{bose2024sestochastic, zeni2025generative}. A wide range of inductive biases has been explored for diffusion models \citep{hoogeboom2022equivariant, falck2025fourier,  shariatian2025denoising, xu2022poisson, xu2023pfgm++, yoon2023scorebased}, but comparable biases have been less studied for flow-based models. One example is spectral bias \citep{falck2025fourier}.

Many natural datasets exhibit power spectra that approximately follow $1/f^\gamma$, with $0 < \gamma \le 2$ \citep{dutta1981low, ivanov1999multifractality, ruderman1997origins,  srivastava2003advances, torralba2003statistics, van1996modelling, voss19781}. More generally, \Figref{fig:1} shows that even when spectra do not follow a power law, they typically decay with increasing frequency. In contrast, diffusion models and flow matching (FM) commonly inject white noise \citep{dieleman2024spectral}, which can be misaligned with the spectral structure of the data \citep{falck2025fourier, randono2025clouddiffusion1theory}. To better match the data's spectral content, we introduce \textbf{Low-Pass Flow Matching} (LP-FM), a variant of FM based on an operator-modulated interpolant. LP-FM induces a time-dependent spectral bias; for small $t$ the path preserves the source spectrum, while as $t \to 1 $ it transitions to a frequency-decaying bias. Empirically, we show that the proposed bias preserves or improves sample quality, can accelerate training, and reduces sampling cost across datasets, including a scientific imaging benchmark.

\section{Related Work}
Recent works analyze generative dynamics through a spectral lens. \citet{liu2023genphys} frame generative processes as physical evolutions satisfying precise spectral properties, while \citet{hoogeboom2023blurring, rissanen2023generative} recast the noising process as heat-equation-based blurring. A growing line of research focuses on spectral data properties \citep{dieleman2024spectral, jiralerspong2025shaping} to optimize noise schedules \citep{benita2025spectral} or to favor specific Fourier components by introducing autocorrelated noise \citep{falck2025fourier, huang2024blue, randono2025clouddiffusion1theory}. Others have introduced generative models that operate directly in general functional spaces \citep{gerdes2024gud, phillips2022spectral, crabbe2024time}. Complementary to this, architectural interventions \citep{si2024freeu, yang2023diffusion, wang2024frequency, wang2025fourierflow} counteract the inherent spectral bias of the underlying network to better synthesize high- and low-frequency components during generation. 

We defer a detailed comparison with \citet{falck2025fourier} and \citet{randono2025clouddiffusion1theory} to \Appref{app:a1}.

\section{Background}
\subsection{Flow Matching}
Let $p$ and $q$ be probability distributions on $\R^d$, where $p$ is a simple source distribution and $q$ is the (unknown) data distribution. FM learns a time-varying vector field $u : [0,1] \times \R^d \to \R^d$ that defines the ODE
\begin{equation}
    \frac{d}{dt}\phi_t(x) = u_t\left(\phi_t(x)\right), 
    \qquad \phi_0(x)=x.
    \label{eq:flowode}
\end{equation}
Setting $p_0 = p$, the ODE induces the probability path $p_t \coloneqq [\phi_t]_{\#}(p_0)$. The goal is to learn $u_t$ such that the terminal distribution $p_1$ matches the data distribution $q$.

In general, $u_t(x)$ is intractable. Instead, \citet{lipman2023flow, tong2024improving} show it is possible to learn $u_t(x)$ by introducing a conditioning variable $z$ and minimizing the conditional flow matching (CFM) objective
\begin{equation*}
    \mathcal{L}_{\text{CFM}}(\theta) \coloneqq \mathbb{E}_{t, \rho(z),p_t(x|z)} \left\|v_\theta(t;x) - u_t(x|z)\right\|^2,
\end{equation*}
where $\rho$ is the distribution over $z$, and $p_t(x|z)$, $u_t(x|z)$ denote the conditional probability path and its associated vector field. Setting $z=x_1$, with $\rho(z)=q(x_1)$, recovers FM \citep{lipman2023flow}, while setting $z=(x_0, x_1)$, with $\rho(z) = p(x_0)q(x_1)$, yields CFM \citep{tong2024improving}.

\subsection{Power Spectra of Diffusion Models}
Consider a random field $x\in\R^{\mathcal{G}}$ on a grid $\mathcal{G}$ with $|\mathcal{G}|=d$, e.g., images with $d=H\times W$.
Let $\hat x(\vk)=\mathcal{F}\{x\}(\vk)$ denote its discrete Fourier transform (DFT).
The power spectral density (PSD) is
\begin{equation*}
    S_x(\vk) \coloneqq \E\!\left[|\hat x(\vk)|^2\right].
\end{equation*}
Assuming $x_0$ and $x_1$ independent, a standard interpolant for diffusion models and FM is
\begin{equation}
    x_t = \alpha_t\,x_1 + \sigma_t\,x_0,
    \qquad x_0 \sim \mathcal{N}(0,I_d),\;\; x_1\sim q.
    \label{eq:standard_interpolant}
\end{equation}
Consequently, the interpolant's PSD is \citep{dieleman2024spectral, falck2025fourier}
\begin{equation*}
    S_{x_t}(\vk)
    = \alpha_t^2 S_{x_1}(\vk) + \sigma_t^2,
\end{equation*}
i.e., the injected noise is additive and constant across frequencies. In the following sections we assume that $x_0$ and $x_1$ are spectrally uncorrelated,
 i.e. $\E[\hat x_0(\vk)\,\overline{\hat x_1(\vk)}]=0$ for all $\vk$, where $\overline{(\cdot)}$ denotes complex conjugation.

\section{Frequency Biased Interpolant for Flow Matching}\label{sec:fbi}
We generalize \Eqref{eq:standard_interpolant} by replacing the scalar noise scaling ($\sigma_t$) with a time-dependent linear shift-invariant (LSI) operator.
Let $\mathcal{L}_t$ be a bounded LSI on $\ell^2(\mathcal{G})$. We define the interpolant
\begin{equation}
    x_t \coloneqq t x_1 + \mathcal{L}_t x_0.
    \label{eq:int}
\end{equation}
We design $\mathcal{L}_t = (1-t)\mathcal{L}_t'$ such that, informally, $\mathcal{L}_t' \to \mathcal{I}$ as $t\to 0$ (where $\mathcal{I}$ is the identity operator), preserving the source bias early in the path, and $\mathcal{L}_t'\to \mathcal{L}$ as $t\to 1$ for a target bias $\mathcal{L}$.

\begin{proposition}\label{prop:psd}
Let $\mathcal{L}_t$ be a bounded LSI operator with frequency response $L_t$ and consider
the interpolant in \Eqref{eq:int}. The PSD is 
\begin{equation*}
    S_{x_t}(\vk)
    = t^2\,S_{x_1}(\vk) + |L_t(\vk)|^2\,S_{x_0}(\vk).
\end{equation*}
\end{proposition}

Proofs of this proposition and the subsequent one are given in \Appref{app:b}. Following \citet{lipman2023flow}, we recover the conditional vector field $u_t(x|x_1)$ from the interpolant in \Eqref{eq:int}.

\begin{proposition}\label{prop:cond_vf}
Let $\mathcal{L}_t$ be differentiable in $t$ and invertible for $t\in[0,1)$. Then, defining $\dot{\mathcal{L}}_t \coloneqq \partial_t \mathcal{L}_t$, the conditional vector field is 
\begin{equation*}
    u_t(x|x_1)= x_1 + \dot{\mathcal{L}}_t\,\mathcal{L}_t^{-1}\!\left(x-tx_1\right).
\end{equation*}
Moreover, if we assume $x_0 \sim \mathcal{N}(\cdot|\vzero, \mI)$, the conditional path is still a Gaussian path
\begin{equation*}
    x_t | x_1 \sim \mathcal{N}\left(\cdot\big| tx_1, \mathcal{L}_t\mathcal{L}_t^\ast\right)
\end{equation*}
where $\mathcal{L}_t^\ast$ is the adjoint of $\mathcal{L}_t$.
\end{proposition}

Finally, following \citet{tong2024improving}, we can condition on $z= (x_0, x_1)$ and obtain the CFM form of the conditional vector field (which we always use in our experiments):
\begin{equation*}
    u_t(x|z) = x_1 + \dot{\mathcal{L}}_t x_0 = x_1 +\mathcal{F}^{-1}\left\{\left(\partial_t L_t(\vk)\right)\hat{x}_0(\vk)\right\}.
\end{equation*}

\textbf{Rotation-Invariant Parameterization.}
For image-like data we use isotropic filters, i.e.\ $L_t(\vk)=L_t(k)$, with radial frequency $ k \coloneqq \sqrt{\|\vk\|^2}$. We consider two low-pass parameterizations: rational low-pass (RLP),
\begin{equation*}
    |L_t(\vk)|^2 = \frac{(1-t)^2}{\left(1 + \frac{k^2}{k_0^2}\right)^{\gamma_1t}},
\end{equation*}
and Gaussian low-pass (GLP),
\begin{equation*}
    |L_t(\vk)|^2 = (1-t)^2 e^{-\gamma_1t\|\vk\|^2},
\end{equation*}
where $k_0>0$ sets the cutoff scale and $\gamma_1>0$ controls the strength of the spectral bias. When used within CFM, we refer to the resulting methods as RLP-CFM and GLP-CFM, respectively.

\section{Experiments}
We evaluate RLP-CFM and GLP-CFM against FM, CFM, and variance-preserving flow matching (VP-FM; \citet{lipman2023flow}) for unconditional image generation. We consider three datasets: Dead Leaves \citep{baradad2021learning}, CIFAR-10, and Galaxy10 DECaLS \citep{galaxy10_github}. Dead Leaves provides a controlled setting with an approximately $1/f^\gamma$ spectrum (cf.\ \Figref{fig:1} and the dashed reference in \Figref{fig:2}), CIFAR-10 is a standard natural-image benchmark, and Galaxy10 represents a real-world scientific imaging task. We report the Fr\'echet Inception Distance (FID) \citep{NIPS2017_8a1d6947} and the number of function evaluations (NFE) required for sampling. FID is computed with \textsc{Clean-Fid} \citep{parmar2021cleanfid} using \texttt{legacy\_tensorflow} settings and training-set Inception statistics. For sampling, we compare several ODE solvers: \textsc{Euler} and adaptive \textsc{Dopri5} \citep{dormand1980family} and \textsc{Tsit5} \citep{tsitouras2011runge}. Our implementation follows the training protocol of \citet{tong2024improving}. We use the U-Net backbone of \citet{dhariwal2021diffusion} with exponential moving average (EMA), dropout, and gradient-norm clipping, and optimize with either Adam \citep{Kingma2014AdamAM} or AdamW \citep{loshchilov2018decoupled}. Full hyperparameters are reported in \Appref{app:c}. Qualitative comparisons between real and generated samples are deferred to \Appref{app:d}.

\textbf{Dead Leaves.} We generate $50$k $1{\times}64{\times}64$ images by placing randomly oriented squares uniformly at random, following \citet{baradad2021learning} to obtain an approximately $1/f^\gamma$ spectrum. We compare CFM and RLP-CFM; results are in \Figref{fig:2}. RLP-CFM achieves substantially lower FID and reaches its best performance earlier: it achieves roughly half the FID of CFM and converges in $\sim$1350 epochs, while CFM continues to improve up to the end of training. The PSD evolution (right panel) suggests that RLP-CFM recovers higher-frequency content earlier along the path and matches the target spectral shape more closely around $t\approx 0.75$.
\begin{figure}[h]
    \centering
    \includegraphics[width=0.9\linewidth]{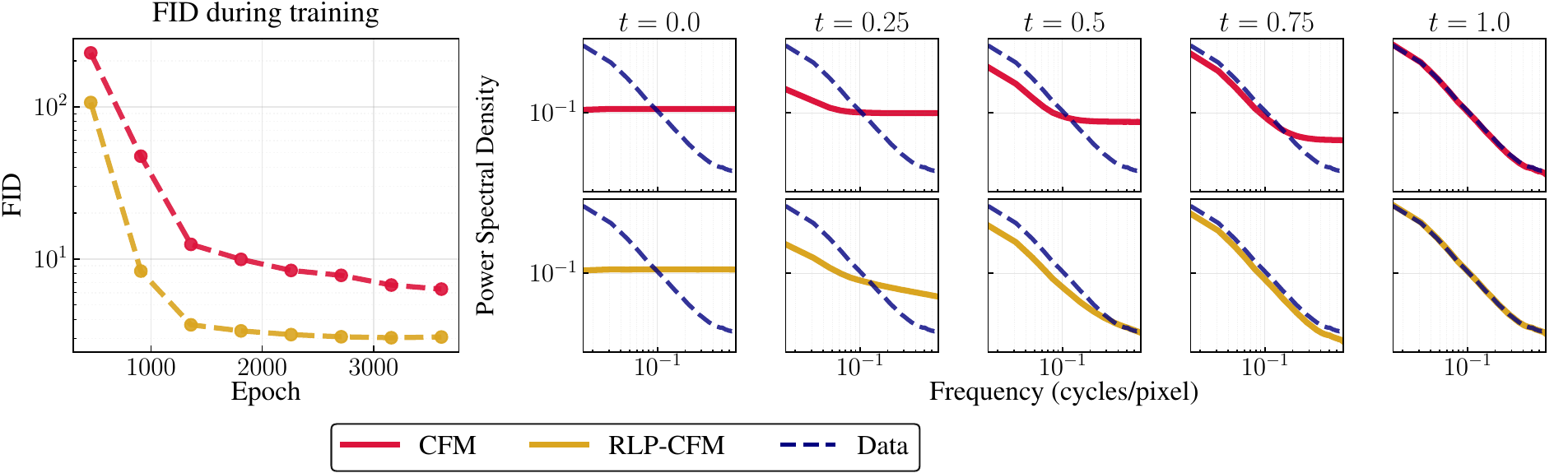}
    \caption{Dead Leaves results for CFM (crimson) and RLP-FM (goldenrod). \figleft FID (12.5k samples), evaluated every 450 training epochs. \figright Radially averaged power spectral density of generated samples at different generation times: top row CFM, bottom row RLP-FM.}
    \label{fig:2}
\end{figure}

\textbf{CIFAR-10 \& Galaxy10 DECaLS.}
 \Tabref{tab:results} summarizes the results (Galaxy10: $16$k RGB images at $256{\times}256$). With \textsc{Euler} solver, baselines are strongest, suggesting that a spectral bias may be less beneficial under a fixed-step discretization\footnote{Analogously, this may apply to DDIM \citep{song2021denoising}; we do not evaluate this here.}. Under adaptive solvers, however, RLP-CFM becomes competitive: on CIFAR-10 it preserves sample quality while reducing NFE, and on Galaxy10 it improves FID while substantially reducing NFE (up to $1.65\times$ vs.\ CFM and $2.16\times$ vs.\ VP-FM). In contrast, GLP-CFM underperforms RLP-CFM on CIFAR-10 and we were unable to obtain stable training on Galaxy10. Overall, these results indicate that a spectral bias can be most useful when paired with adaptive integration.
\begin{table}[h]
  \centering
  \renewcommand{\arraystretch}{1.12}
  \setlength{\tabcolsep}{6.5pt}
  \begin{tabular}{l c c c | c c c}
  \toprule
  Dataset $\rightarrow$
    & \multicolumn{3}{c}{CIFAR-10}
    & \multicolumn{3}{c}{Galaxy10 DECaLS} \\
  \cmidrule(lr){2-4}\cmidrule(lr){5-7}
  Solver $\rightarrow$
    & Euler (NFE=100)
    & \multicolumn{2}{c}{Adaptive}
    & Euler (NFE=100)
    & \multicolumn{2}{c}{Adaptive} \\
  \cmidrule(lr){3-4}\cmidrule(lr){6-7}
  Method $\downarrow$
    & FID
    & FID
    & NFE
    & FID
    & FID
    & NFE \\
  \midrule
  CFM/FM & 4.66 & \textbf{3.96} & 80.7 & \textbf{11.94} & 11.23 & 92.0 \\
  VP-FM  & \textbf{4.44} & 4.29 & 101.2 & 12.42   & 12.84   & 120.75 \\
  \midrule
  RLP-CFM   & 5.27 & 4.09 & 67.5 & 13.78 & \textbf{10.03} & \textbf{56.0} \\
  GLP-CFM   & 6.71 & 5.35 & \textbf{59.1} & --    & --    & -- \\
  \bottomrule
  \end{tabular}
\caption{FID and NFE with \textsc{Euler} and adaptive solvers (\textsc{Dopri5}/\textsc{Tsit5}; best over adaptive). Baselines: CFM, VP-FM. This work: RLP-CFM, GLP-CFM.}
\label{tab:results}
\end{table}
\section{Conclusion}
We introduced Low-Pass Flow Matching, a simple operator-modulated interpolant that induces a time-varying spectral bias in Flow Matching. Experiments suggest the bias is most useful with adaptive ODE solvers, improving or preserving sample quality while reducing sampling cost. Future work will evaluate RLP-CFM on additional scientific datasets. \newpage

%%%%%%%%----- BIBLIOGRAPHY ----------%%%%%%%%%%%

\bibliography{iclr2026_delta}
\bibliographystyle{iclr2026_delta}\newpage

%%%%%%%%----- APPENDIX ----------%%%%%%%%%%%

\appendix
\section{Extendend Related Work}
\label{app:a}

\textbf{Inductive biases in generative modeling.} Generative models often rely on inductive biases to tailor their sampling dynamics to specific data structures. Notable examples include physics-inspired biases, such as L\'evy processes for heavy-tailed data \citep{yoon2023scorebased, shariatian2025denoising} and Poissonian flows \citep{xu2022poisson, xu2023pfgm++}, that have been found to be a valid alternative to diffusion models. Another deeply explored area is geometric bias; \citet{kohler2020equivariant, garcia2021n} developed equivariant continuous normalizing flows, while \citet{hoogeboom2022equivariant, bose2024sestochastic} demonstrate the necessity of symmetry-preserving models for biological data. Beyond symmetry, a broad body of work covers alternative inductive biases in generative dynamics, including methods that design probability path or flows with low-curvature trajectories \citep{lipman2023flow, tong2024improving, liu2023flow}.

\subsection{Comparison with EqualSNR and Cloud Diffusion}
\label{app:a1}
We compare LP-FM with EqualSNR and Cloud Diffusion, two methods that modify diffusion models by using autocorrelated noise.

\textbf{Cloud Diffusion} \citep{randono2025clouddiffusion1theory} considers noise spectra of the form $S_{\text{noise}}(k)\propto 1/k$, which diverge as $k\to 0$ and may therefore lead to numerical instability near the DC component; since no experiments are reported, we limit our discussion to methodological considerations.

\textbf{EqualSNR} \citep{falck2025fourier} instead sets $S_{\text{noise}}(\vk)\propto S_1(\vk)$ to align the forward noising spectrum with the data spectrum, motivated by Gaussian assumptions in the DDPM reverse process \citep{ho2020denoising}. In contrast, FM does not require the same reverse-time Gaussianity; consequently, it is a priori unclear whether this rationale transfers. This motivates studying spectral bias in FM in its own right. Although EqualSNR is motivated by DDPM assumptions, it is evaluated with DDIM sampling \citep{song2021denoising}, which is closely related to FM with \textsc{Euler} solver \citep{gao2025diffusionmeetsflow}. Empirically, we do not observe benefits from introducing a spectral bias under \textsc{Euler} (except on the toy Dead Leaves), whereas gains emerge when FM is paired with adaptive ODE solvers. Therefore, the benefits we report arise under sampling conditions substantially different from DDIM. There are also methodological differences: EqualSNR operates in Fourier space, whereas we work in pixel space.

Finally, our approach induces a time-varying bias: $\mathcal{L}_t$ is designed to preserve the source spectral structure early along the path while progressively imposing a low-pass bias near the data distribution. In contrast, EqualSNR and Cloud Diffusion prescribe a noise spectrum intended to match the data, but do not explicitly preserve the spectral bias of the source distribution. Overall, our results highlight that (i) the solver choice can be critical for realizing the benefits of spectral bias in FM, and (ii) time-varying operator modulation provides a way to incorporate spectral structure while remaining compatible with nontrivial source distributions.

\section{Proofs}
\label{app:b}

This appendix collects proofs of the propositions in \Secref{sec:fbi}. For convenience, we restate each proposition before its proof. We also recall a few
operator-theoretic and Fourier-analytic facts used throughout.

\subsection{Preliminaries: LSI operators on a finite grid}

Let $\mathcal{G}$ be a finite periodic grid,
and equip $\mathbb{C}^{\mathcal G}$ with the standard $\ell^2$ inner product $\langle x,y\rangle \coloneqq \sum_{g\in\mathcal G} x(g)\,\overline{y(g)}$ and norm $\|x\|_2^2=\langle x,x\rangle$.
We denote by $\hat x(\vk)=\mathcal F\{x\}(\vk)$ the DFT of $x$.
We use the convention that the DFT is unitary so that Parseval's identity holds:
\begin{equation*}
    \|x\|_2^2 = \sum_{\vk} |\hat x(\vk)|^2.
\end{equation*}

\paragraph{Bounded linear operators.}
A linear map $\mathcal L:\ell^2(\mathcal G)\to \ell^2(\mathcal G)$ is bounded if
\[
\|\mathcal L\|_{\mathrm{op}} \coloneqq \sup_{\|x\|_2=1}\|\mathcal L x\|_2 < \infty.
\]
On a finite-dimensional space every linear operator is bounded.

\paragraph{Shift-invariant (LSI) operators and frequency response.}
An operator $\mathcal L$ is shift-invariant if it commutes with grid shifts.
Equivalently, $\mathcal L$ is a circular convolution with some kernel $h \in\R^{\mathcal G}$:
\begin{equation*}
    (\mathcal L x)(g) = (h * x)(g).
\end{equation*}
Such operators diagonalize in the Fourier domain: there exists a complex-valued function
$L(\vk)$ (the frequency response) such that for all $\vk$,
\begin{equation}
    \label{eq:operator_frequency}
    \mathcal{F}\left\{\mathcal L x\right\}(\vk) = L(\vk)\,\hat x(\vk)
\end{equation}
For a time-indexed family $\{\mathcal L_t\}_{t\in[0,1]}$ we write $L_t(\vk)$ for its response.

\paragraph{Invertibility.}
An LSI operator $\mathcal L$ is invertible on $\ell^2(\mathcal G)$ if and only if
$L(\vk)\neq 0$ for all $\vk$. In that case the inverse is also LSI and satisfies $\mathcal{F}\left\{\mathcal L^{-1} x\right\}(\vk) = \frac{1}{L(\vk)}\,\hat x(\vk)$.

\paragraph{Spectral (un)correlatedness.}
For random fields $x_0,x_1\in\R^{\mathcal G}$ we say they are spectrally uncorrelated if
\[
\E\!\left[\hat x_0(\vk)\,\overline{\hat x_1(\vk)}\right]=0
\qquad \text{for all }\vk.
\]

\subsection{Proof of \Propref{prop:psd}}

\begin{proposition}\label{prop:1}
Let $\mathcal L_t$ be a bounded LSI operator with frequency response $L_t(\vk)$ and consider
\[
x_t = t\,x_1 + \mathcal L_t x_0.
\]
Assume $x_0$ and $x_1$ are spectrally uncorrelated.
Then the PSD of $x_t$ satisfies
\[
S_{x_t}(\vk)=t^2\,S_{x_1}(\vk)+|L_t(\vk)|^2\,S_{x_0}(\vk).
\]
\end{proposition}

\begin{proof}[Proof of \Propref{prop:1}]
Taking the DFT of the interpolant and using linearity yields, for each $\vk$,
\begin{equation*}
    \hat x_t(\vk)= t\,\hat x_1(\vk) + \mathcal{F}\left\{\mathcal L_t x_0\right\}(\vk).
\end{equation*}
Since $\mathcal L_t$ is LSI with response $L_t(\vk)$, \Eqref{eq:operator_frequency} yields
$\mathcal{F}\{\mathcal L_t x_0\}(\vk)=L_t(\vk)\,\hat x_0(\vk)$, hence
\[
\hat x_t(\vk)= t\,\hat x_1(\vk) + L_t(\vk)\,\hat x_0(\vk).
\]
Therefore,
\begin{align*}
S_{x_t}(\vk)
&= \E\!\left[|\hat x_t(\vk)|^2\right]
= \E\!\left[\left|t\,\hat x_1(\vk)+L_t(\vk)\,\hat x_0(\vk)\right|^2\right] \\
&= t^2\,\E\!\left[|\hat x_1(\vk)|^2\right]
  + |L_t(\vk)|^2\,\E\!\left[|\hat x_0(\vk)|^2\right]
  + 2\,\Re\!\left(t\,L_t(\vk)\,\E\!\left[\hat x_0(\vk)\,\overline{\hat x_1(\vk)}\right]\right).
\end{align*}
The cross term vanishes by spectral uncorrelatedness.
Recognizing $\E[|\hat x_i(\vk)|^2]=S_{x_i}(\vk)$ for $i\in\{0,1\}$ yields
\begin{equation*}
    S_{x_t}(\vk)=t^2\,S_{x_1}(\vk)+|L_t(\vk)|^2\,S_{x_0}(\vk).
\end{equation*}
\end{proof}

\subsection{Proof of \Propref{prop:cond_vf}}

\begin{proposition}\label{prop:2}
Let $\{\mathcal L_t\}_{t\in[0,1)}$ be a family of bounded, invertible LSI operators, and assume that
$t\mapsto \mathcal L_t$ is differentiable with respect to the operator norm. Define $\dot{\mathcal L}_t\coloneqq \partial_t \mathcal L_t$ and consider the interpolant
\begin{equation*}
    x_t = t\,x_1 + \mathcal L_t x_0.
\end{equation*}
Then the conditional vector field $u_t(x|x_1)$ associated with the conditional path
$p_t(\cdot|x_1)$ is
\begin{equation}\tag{a}
    u_t(x|x_1)= x_1 + \dot{\mathcal L}_t\,\mathcal L_t^{-1}\!\left(x-tx_1\right).
    \label{eq:a}
\end{equation}
Moreover, if $x_0\sim\mathcal{N}(\cdot|\vzero, \mI)$, then
\begin{equation}\tag{b}
    x_t | x_1 \sim \mathcal{N}\left(\cdot\big| tx_1, \mathcal{L}_t\mathcal{L}_t^\ast\right)
    \label{eq:b}
\end{equation}
where $\mathcal{L}_t^\ast$ is the adjoint of $\mathcal{L}_t$.
\end{proposition}

\begin{proof}[Proof of \Propref{prop:2}]
We prove \Eqref{eq:a} and \Eqref{eq:b}:
\begin{enumerate}
    \item[(a)] Our derivation proceeds along the the lines of \citet{lipman2023flow}, who derive the conditional vector field for a Gaussian path with scalar noise scaling.
    
    Fix $x_1$. Define the family of conditional flow maps $\phi_t(\cdot| x_1):\R^{\mathcal G}\to\R^{\mathcal G}$ by
    \begin{equation*}
        \phi_t(x_0| x_1)\coloneqq t\,x_1+\mathcal L_t x_0.
    \end{equation*}
    For each fixed $x_1$, the conditional path $p_t(\cdot|x_1)$ is the pushforward of the source law of $x_0$ through $\phi_t(\cdot|x_1)$, i.e. $x_t|x_1=\phi_t(x_0|x_1)$. 
    
    Let $y=\phi_t(x_0| x_1)$. Since $\mathcal L_t$ is invertible for all $t\in[0,1)$, the map $\phi_t(\cdot|x_1)$
    is invertible and
    \begin{equation*}
        x_0=\phi_t^{-1}(y|x_1)=\mathcal L_t^{-1}(y-tx_1).
    \end{equation*}
    Differentiating $\phi_t(x_0|x_1)$ with respect to $t$ yields
    \begin{equation*}
        \partial_t \phi_t(x_0|x_1)=x_1+\dot{\mathcal L}_t x_0.
    \end{equation*}
    By definition of the velocity field associated with the conditional flow,
    \begin{equation*}
        u_t(y|x_1)=\partial_t \phi_t(x_0| x_1)\Big|_{x_0=\phi_t^{-1}(y|x_1)}
    = x_1+\dot{\mathcal L}_t\,\mathcal L_t^{-1}(y-tx_1).
    \end{equation*}
    Renaming $y$ as $x$ gives \Eqref{eq:a}.
    \item[(b)] Assume $x_0\sim\mathcal N(\cdot\mid \vzero, \mI)$ and fix $x_1$. Since
    \[
    x_t = t\,x_1+\mathcal L_t x_0
    \]
    is an affine transformation of the Gaussian random vector $x_0$, it is Gaussian. Its conditional mean is
    \begin{equation*}
        \E[x_t|x_1]=t\,x_1+\mathcal L_t\,\E[x_0]=t\,x_1,
    \end{equation*}
    and its conditional covariance is
    \begin{align*}
        \mathrm{Cov}(x_t|x_1) &= \E\!\left[\bigl(x_t-\E[x_t| x_1]\bigr)\bigl(x_t-\E[x_t|x_1]\bigr)^\ast \,\big|\, x_1\right] \\
        &= \E[x_tx_t^\ast|x_1] - \E[x_t |x_1]\,\E[x_t |x_1]^\ast \\
        &= \E\!\left[\bigl(\mathcal L_t x_0\bigr)\bigl(\mathcal L_t x_0\bigr)^\ast \,\big|\, x_1\right]\\
        &= \mathcal L_t\,\E[x_0 x_0^\ast]\,\mathcal L_t^\ast \\
        &= \mathcal L_t\,\mathrm{Cov}(x_0)\,\mathcal L_t^\ast \qquad\qquad\quad \\
        &= \mathcal L_t\,I_d\,\mathcal L_t^\ast \\
        &= \mathcal L_t\mathcal L_t^\ast.
    \end{align*}
    Therefore,
    \begin{equation*}
        x_t\mid x_1 \sim \mathcal N\!\left(\cdot\mid t x_1,\ \mathcal L_t\mathcal L_t^\ast\right),
    \end{equation*}
    which is \Eqref{eq:b}
\end{enumerate}
\end{proof}

\newpage
\section{Experimental Settings}
\label{app:c}

\begin{table}[h]
\centering
\renewcommand{\arraystretch}{1.25}
\setlength{\tabcolsep}{4pt}
\begin{tabular}{l c c c c c}
\toprule
Dataset $\downarrow$
& Channels
& Depth
& Channel multipliers
& Attention resolutions
& Head channels\\
\midrule
Dead Leaves
& 128
& 2
& [1,2,2,2]
& 16
& 64 \\
CIFAR-10
& 128
& 2
& [1,2,2,2]
& 16
& 64 \\
Galaxy10
& 128
& 2
& [1,1,2,2,2,4]
& 32,16,8
& 64 \\
\bottomrule
\end{tabular}
\caption{UNET settings for each dataset. For more information, please refer to \citet{dhariwal2021diffusion}.}
\end{table}

\begin{table}[h]
\centering
\renewcommand{\arraystretch}{1.25}
\setlength{\tabcolsep}{4pt}
\begin{tabular}{l c c c c | c c c}
\toprule
Dataset $\downarrow$
& Epochs
& EMA
& Dropout
& Clip grad norm
& Optimizer
& LR
& LR Warmup \\
\midrule
Dead Leaves
& 3750
& 0.9999
& 0.1
& 1.0
& AdamW
& $1 \times 10^{-4}$
& \ding{55} \\
CIFAR-10
& 1900
& 0.9999
& 0.1
& 1.0
& Adam
& $2 \times 10^{-4}$
& \ding{55} \\
Galaxy10
& 3750
& 0.9999
& 0.05
& 0.75
& Adam
& $1 \times 10^{-4}$
& \ding{51} \\
\bottomrule
\end{tabular}
\caption{Training hyperparameters used for each dataset (training epochs, EMA decay, dropout, gradient clipping norm, optimizer, learning rate, learning rate warmup).}
\end{table}

\section{Generated images}
\label{app:d}
\begin{figure}[h!]
    \centering
    \includegraphics[width=1\linewidth]{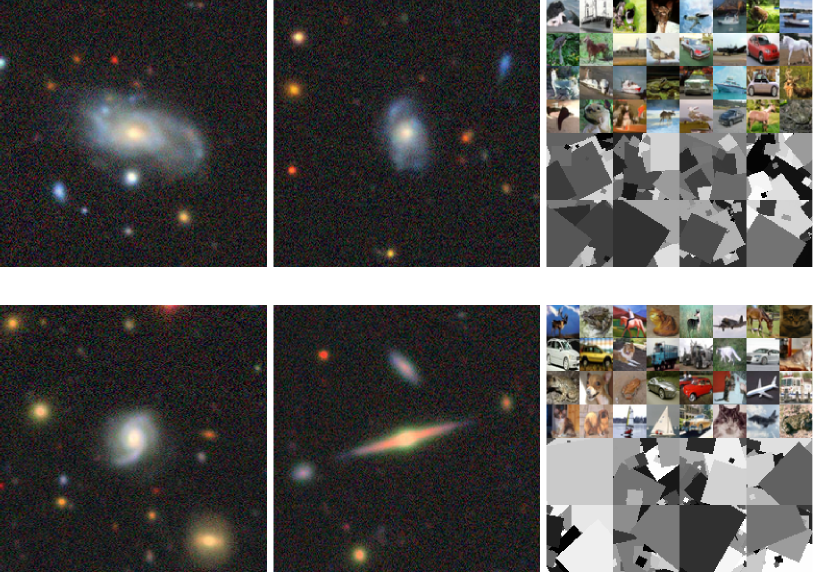}
    \caption{Galaxy10 DECaLS, CIFAR-10, and Dead Leaves images. \figtop Generated images with RLP-CFM. \figbottom Real images.}
    \label{fig:3}
\end{figure}

\end{document}